\documentclass[letterpaper, 10 pt, conference]{ieeeconf}  

\IEEEoverridecommandlockouts                              
\usepackage{amssymb,gensymb}
\usepackage{algorithm}
\usepackage{color}
\usepackage{multirow,multicol}
\usepackage{booktabs,pgfplots}
\usepackage{subcaption}

\usepackage{amsmath,epstopdf}
\usepackage[noend]{algpseudocode}
\usepackage[noadjust]{cite}

\newtheorem{lemma}{Lemma}
\newtheorem{theorem}{Theorem}

\begin{document}

\title{\LARGE \bf
Visual Monitoring for Multiple Points of Interest on a 2.5D Terrain using a UAV  with Limited Field-of-View Constraint}

\author{Parikshit Maini, P.B. Sujit, and Pratap Tokekar
\thanks{P. Maini and P.B. Sujit are with Indraprastha Institute of Information Technology, India. \texttt{\small \{parikshitm, sujit\}@iiitd.ac.in}}
\thanks{P. Tokekar is with the Department of Electrical \& Computer Engineering, Virginia Tech, USA. \texttt{\small tokekar@vt.edu}}%
}

\maketitle
\begin{abstract}
 Varying terrain conditions and limited field-of-view restricts the visibility of aerial robots while performing visual monitoring operations. In this paper, we study the multi-point monitoring problem on a 2.5D terrain using an unmanned aerial vehicle (UAV) with limited camera field-of-view. This problem is NP-Hard and hence we develop a two phase strategy to compute an approximate tour for the UAV. In the first phase, visibility regions on the flight plane are determined for each point of interest. In the second phase, a tour for the UAV to visit each visibility region is computed by casting the problem as 
 an instance of the Traveling Salesman Problem with Neighbourhoods (TSPN). We design a constant-factor approximation algorithm for the TSPN instance. Further,  
 we reduce the TSPN instance to an instance of the Generalized Traveling Salesman Problem (GTSP) and devise an ILP formulation to solve it. We present a comparative evaluation of solutions computed using a branch-and-cut implementation and an off-the-shelf GTSP tool -- GLNS, while varying the points of interest density, sampling resolution and camera field-of-view. We also show results from preliminary field experiments.

\end{abstract}
\section{Introduction}\label{sec:intro}

Visual surveillance and monitoring is an important application area for aerial robots. Crop management \cite{tokekar2016precision}, area coverage \cite{araujo2013coverage,mainiRefuelUAV}, terrain mapping \cite{obermeyer2012sampling}, structural inspection \cite{morgenthal2014quality}, and disaster management \cite{erdelj2017help,sujit2007cooperative} are some applications where aerial robots are widely used. When planning paths in such missions, it is important to take visibility obstructions into account. Landscape features such as mountains, gorges, buildings, and bridges limit the line-of-sight of the aerial robots. In addition, operative limitations such as camera field-of-view and maximum flight altitude corresponding to the image resolution and/or regulatory requirements also restrict visibility. It is imperative that such restrictions be accounted for when planning for monitoring missions.

In this paper, we address the visual monitoring problem on 2.5D terrains using a UAV while accounting for camera field-of-view and terrain imposed visibility restrictions. We present a two-phase strategy to compute a tour for an aerial robot to visually monitor a set of points located within a terrain. A naive strategy is to visit each point of interest (or a point directly above it). This strategy does not exploit the camera field-of-view and essentially assumes the most restrictive field-of-view only along the center of the camera. When considering the flight altitude, this can lead to solutions that are numerically far from the optimal by a factor of $R$, where $R$ is the radius of the camera footprint on the ground (assuming a circular camera footprint). In fact, in the special case that the field-of-view angle of the camera sensor tends to zero, the problem reduces to an instance of the Traveling Salesman Problem (TSP) that is known to be NP-hard~\cite{mitchell1999guillotine}. Since our problem is a generalization of TSP, it is at least as hard as TSP.


\section{Related Work}\label{sec:relWork}
Visual monitoring and surveillance using UAVs has been an active area of research over the past decade. Various lines of work have addressed aspects related to target monitoring with differential priorities \cite{priorityBasedRouting}, multi-robot surveillance \cite{tokekar2016precision,mainiRefuelUAV,maini2018persistent,maini2018visibility}, persistent monitoring \cite{maini2018persistent,smithPersistent2015,mainiRefuelUAV}, mission planning while addressing robot kinematics \cite{mainiASCC2017,xu2011optimal} and so on. There has been limited effort towards addressing the visual monitoring problem in the presence of visibility restrictions due to terrain features. Terrain visibility however is a classic problem in computational geometry \cite{regionIntervisibility} and graphics literature (\cite{visibilitySurvey2003} and references within). Terrain guarding \cite{regionIntervisibility}, and watchtower problems \cite{watchtower_pankaj} relate to computing a set of points that lie on the terrain and at an altitude, respectively, to ensure line-of-sight area coverage on terrains. The problem of line-of-sight coverage within a polygon by a watchman (or a robot), known as the Watchman Routing Problem (WRP), is also well studied in the literature (\cite{tan2001fast} and the references within). Variations of WRP, including the homogeneous \cite{Carlsson1999,tokekar2015persistent,histogram} and heterogeneous \cite{maini2018visibility,maini2018persistent} multiple-robot versions (both within restricted sub-domains) have also been studied extensively. Recently, Maini et. al. \cite{maini2018visibility,maini2018persistent} modeled the coverage problem for piece-wise linear features within terrains as a variant of the n-WRP. 

There is a lot of literature within the aerial robotics community on coverage path planning. Area decomposition based on camera footprint and/or obstacle-free space is a popular choice and admits a robust discretization of the area of interest
\cite{araujo2013coverage,hameed2016side,xu2011optimal,nam2016approach,choi3D}. Other techniques include seed-spreader algorithms \cite{xu2011optimal}, potential fields \cite{hameed2016side} and graph-based search algorithms \cite{nam2016approach}. However, most of the existing works on coverage path-planning assume a flat surface and do not account for altitude variance (and hence the visibility obstructions) of the ground surface. A closely related work is that of Choi et. al. \cite{choi3D}, who address a constant resolution coverage problem that takes into account camera viewing direction and altitude to maintain the image resolution. In this work, we address a multiple-point monitoring problem using an aerial robot while explicitly accounting for visibility restrictions due to the shape of the terrain and camera field of view. Main contributions of this paper are as follows:
\begin{itemize}
   \item Extraction of visibility regions and modeling the path-planning problem as an instance of TSPN
    \item Design of a constant-factor approximation algorithm to solve the class of TSPN instances encountered within the path planning problem
    \item Reduction of the TSPN instance to GTSP thus allowing the application of existing algorithmic tools for GTSP. A new ILP formulation and a branch-and-cut implementation to solve GTSP
    \item Validation of the developed methods in simulation and field experiments
\end{itemize}

The rest of the paper is organized as follows. The application scenario and a formal problem definition are developed in Section \ref{sec:appScenario}. In Section \ref{sec:visComputation}, we describe a method to compute visibility regions on the constant altitude flight plane for the points of interest. In Section \ref{sec:apxAlgo} we develop a constant-factor approximation algorithm for the class of TSPN instances encountered within the path planning problem. Section \ref{sec:routePlanning} outlines an ILP formulation and a branch-and-cut implementation to solve GTSP. Sections \ref{sec:simResults} and \ref{sec:expts} discuss evaluation results in simulation and field trials, respectively. Concluding remarks and future directions are identified in Section \ref{sec:concFut}.
\section{Problem Formulation}\label{sec:appScenario}


Consider an environment $\mathcal E$ and a set of points of interest, $\mathcal P=\{p_i: i\in [1,m]\}$, as shown in Figure \ref{fig:sampleEnv}. We represent the topographical surface within $\mathcal E$ using a polyhedral (or 2.5D) terrain (\cite{sack1999handbook}, pg 352) and model it as a  triangular irregular network (TIN). 
Let $h$ be the constant flight altitude for UAV operation (higher than the altitude of all terrain features). We assume that the UAV is equipped with a fixed downward-facing camera having a constant focal length and a circular field-of-view (FOV). The camera casts a conical field of view on the terrain, as shown in Figure \ref{fig:camerCone}. 

While our strategies are extensible to the case of varying flight altitude, as may be required for constant resolution flights, we assume a constant flight altitude for ease of exposition. We briefly discuss the extension to constant resolution (varying altitude) flight operations, later in the text (Section \ref{sec:apxAlgo}). The circular FOV is also a non-binding assumption and  can be easily extended to a non-circular FOV. The FOV may then be represented as a fixed view angle, $\delta$, in each direction. 

\begin{figure}
\centering
\begin{subfigure}{0.48\linewidth}
\centering
\includegraphics[scale = 0.4]{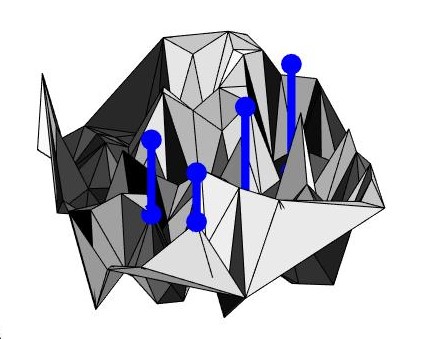}
\caption{}\label{fig:sampleEnv}
\end{subfigure}
\begin{subfigure}{0.48\linewidth}
\centering
\includegraphics[scale = 0.6]{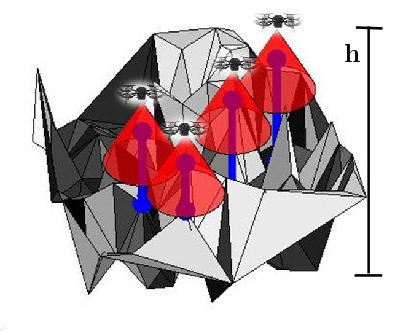}
\caption{}\label{fig:camerCone}
\end{subfigure}
\caption{(a) A sample 2.5 D terrain showing a set of points of interest marked with blue spikes. (b) Monitoring task using aerial robot camera sensor at altitude $h$. Camera FOV shown in red conical regions.}
\end{figure}

We introduce the following UAV routing problem on terrains (URPT), defined as: \emph{Given an environment with a polyhedral terrain, a set of points $\mathcal P$ located on the terrain, a UAV comprising of a fixed down-facing camera with a view angle $\delta$ operating at a fixed altitude $h$ above ground level, plan a minimum length tour for the UAV to  visually monitor all points in the set $\mathcal P$.}

We develop a two phase solution strategy for URPT. The first phase computes a visibility region for each point of interest, $p_i \in \mathcal P$. Visibility region of a terrain point is a closed connected space on the constant altitude flight plane, within the aerial robot's operational region that allows the robot to monitor the corresponding point and may assume a complex geometric shape due to obstructive features within the terrain (Figure \ref{fig:region}). We outline a method to compute the visibility region for each point of interest in Section \ref{sec:visComputation}. Second phase involves route planning for the aerial robot to visit each of the visibility regions to complete a monitoring mission. Sections \ref{sec:apxAlgo} and \ref{sec:routePlanning} develop route planning methods for the aerial robot.




\section{Visibility Computation} \label{sec:visComputation}

\begin{figure}
\centering
\begin{subfigure}{0.58\linewidth}
\centering
\includegraphics[width = \textwidth,height= 2.4cm]{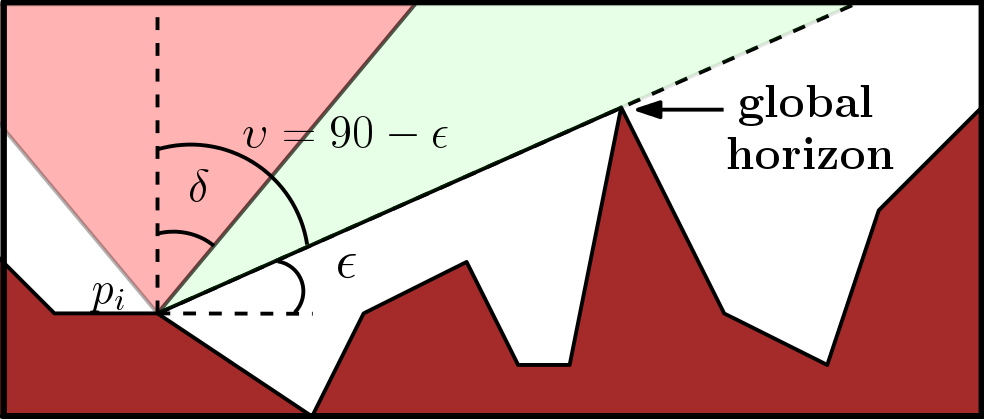}
\caption{}\label{fig:horizon}
\end{subfigure}
\begin{subfigure}{0.4\linewidth}
\centering
\includegraphics[trim = {0cm 2mm 0 0},scale = 0.2]{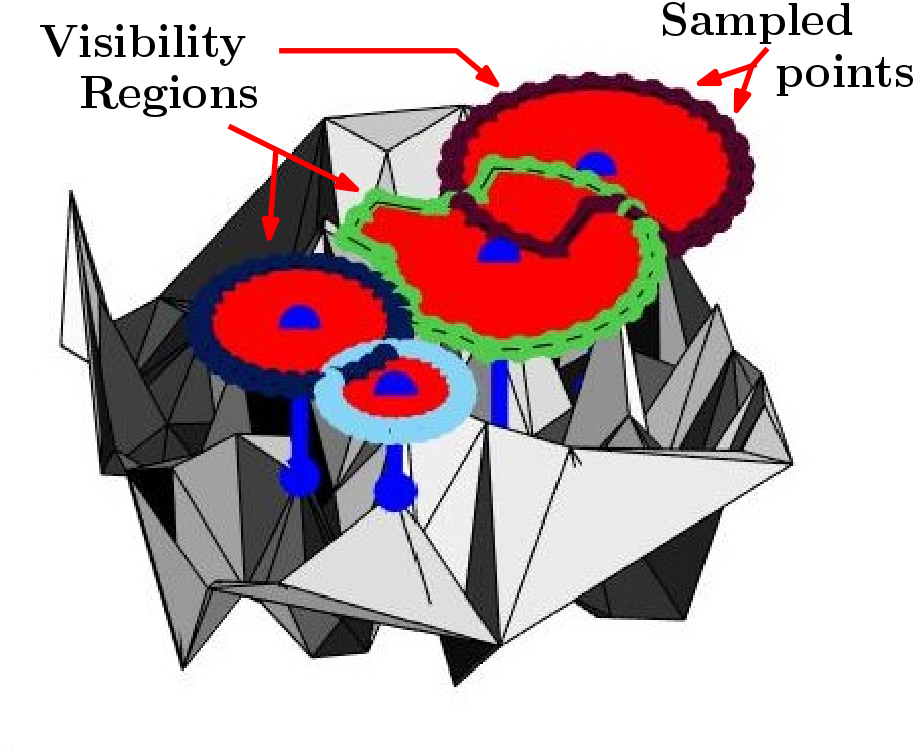}
\caption{}\label{fig:region}
\end{subfigure}
\caption{(a) Global horizon point is the farthest visible point in a radial direction. $\epsilon$ is the elevation of the horizon point, $\upsilon$ is the visibility angle and $\delta$ is the camera view angle. The minimum of $\upsilon$ and $\delta$ determines the boundary of the visibility region in a given radial direction. (b) Visibility regions (red closed shapes) on the constant altitude plane for a set of points on the terrain. Visibility region boundary for each point is marked in a different color for ease of distinguishing. The boundary for a visibility region is computed as the linear interpolation of the projections in $d$ radial directions.}\label{fig:visRegion}
\end{figure}

Consider a point, $p_i \in \mathcal{P}$, that needs to be monitored as shown in Figure \ref{fig:horizon}. To compute visibility region for $p_i$ on the terrain, we place a viewpoint at $p_i$ and compute terrain visibility. The farthest visible point in a given radial direction, that obstructs all points beyond itself when viewed from a specific viewpoint is called \emph{global horizon point} \cite{visibilitySurvey2003} and is expressed in terms of the elevation angle $\epsilon$. The locus of all global horizon points forms the global horizon. \emph{Visibility angle} $\upsilon$ in a radial direction is computed as the complement of the elevation angle of the global horizon point. The minimum of visibility angle and camera view angle defines the boundary of the visibility region in a radial direction.

Horizon computation is a fairly well studied problem in the graphics literature (\cite{visibilitySurvey2003} and references within). We employ the approximate horizon computation method developed by Stewart \cite{horizonComputation}. However, other methods in the literature may also be used since our solution approach is independent of the horizon computation algorithm used. 
 Our selection of Stewart's algorithm was motivated due to its ease of implementation and computational tractability. The interested reader may refer \cite{horizonComputation} for details on the algorithm. 
 
 We use Stewart's algorithm on the terrain shown in Figure \ref{fig:sampleEnv} to compute the horizon at each point of interest. The radial space is sampled in a discrete number of values, $d$\footnote{Sampling resolution, $d$, is a user-input parameter and relates to computational complexity of the approach. Its effects are discussed in more detail in Sections \ref{sec:routePlanning} and \ref{sec:results}.} and the minimum of $\upsilon$ and $\delta$ is computed in each direction. Linear interpolation of the extended projections in each direction on the constant altitude plane at height $h$ is then used to compute the boundary of the visibility region as shown in Figure \ref{fig:region}.  

\section{Polynomial-Time Approximation Algorithm for Route Planning}\label{sec:apxAlgo}

The route planning problem is a generalization of the NP-hard Traveling Salesperson Problem ~\cite{mitchell1999guillotine}. As a result, finding the optimal solution in polynomial time is impossible (unless $P=NP$). In this section, we present a polynomial time approximation algorithm for route planning. Specifically, we present a polynomial-time algorithm that finds a tour for the UAV whose length is guaranteed to be within a constant-factor of the minimum length.

The input to our algorithm is the set of visibility regions that are computed using the method described in Section~\ref{sec:visComputation}. The problem of finding the shortest tour that visits a set of 2D regions is known as the TSP with neighborhoods (TSPN). The neighborhoods correspond to the visibility regions, in our case. TSPN is NP-hard. However, there exists polynomial-time approximation algorithms for many special cases such as when the neighborhoods are all disks of the same radii~\cite{dumitrescu2003approximation} and non-overlapping convex polygons~\cite{mitchell2010constant}. These regions may not necessarily be polygonal (may contain circular arcs) or convex and can be overlapping. Nevertheless, we show how to approximate the visibility regions by possibly-overlapping disks of the same radius. We then show that this approximation still yields a tour whose length is bounded with respect to the optimal.

In the following, let $V=\{V_i: i\in [1,m]\}$ be the set of input visibility regions corresponding to the points of interest $\mathcal P$ that the robot must monitor.

\paragraph*{Lower Bound} We start by showing a lower bound on the length of the optimal (unknown) tour. Recall that $h$ is the height of the fixed-altitude plane on which the robot is allowed to fly and $\delta$ is the FOV angle. We construct a lower-bound approximation tour for the optimal one as follows. Replace each $V_i$ by a disk, say $D_i$, whose radius is equal to $h\tan\delta$. Let $D$ denote the collection of all the disks $D_i$. The disk, $D_i$ lies in the constant altitude flight plane at the height $h$ and centered at the same $x$ and $y$ coordinates as that of $p_i$. 

\begin{lemma}
The visibility region $V_i$ is completely contained within the disk $D_i$. 
\label{lem:contained}
\end{lemma}
\begin{proof}
Recall that $V_i$ is obtained by projecting a reverse cone whose apex is at $p_i$ on the fixed altitude plane at height $h$.  Let the coordinates of $p_i$ be $(x_i,y_i,z_i)$. Consider a reverse cone drawn centered at $(x_i,y_i,0)$. Further assume that this cone is not obstructed by any point on the terrain. It is clear that this cone completely contains the cone drawn at $p_i$. The intersection of the larger cone with the fixed altitude plane at height $h$ yields the disk $D_i$. Therefore, $V_i$ is completely contained within $D_i$. (In the extreme case, $V_i$ is the same as $D_i$.)
\end{proof}

\begin{lemma}
Let $L_V^{*}$ be the length of the optimal tour that visits at least one point in each visibility region, $V_i\in V$. Let $L_D^{*}$ be the length of the optimal tour that visits at least one point in each disk, $D_i\in D$. We have: $L_D^{*} \leq L_V^{*}$.
\label{lem:lengthlowerbound}
\end{lemma}
\begin{proof}
From Lemma~\ref{lem:contained}, we know that $V_i\subseteq D_i$. Therefore, any tour that visits at least one point in each $V_i$ is also a tour that visits at least one point in each $D_i$. As a result, the optimal tour (of length $L_V^{*}$) that visits at least one point in each $V_i$ is also a tour that visits at least one point in each $D_i$. However, $L_D^{*}$ is the length of the optimal tour that visits at least one point in each $D_i$. Therefore, $L_D^{*} \leq L_V^{*}$.
\end{proof}

Finally, we lower bound the length of the optimal tour that visits at least one point in each $V_i$. We relate the lower bound to the maximum number of non-overlapping disks. Specifically, let $D^I\subseteq D$ be the largest set of disks, $D_i$, such that no two disks overlap with each other. This can be found out greedily by constructing the maximum independent set of the disks, as shown in~\cite{dumitrescu2003approximation}. Let $m$ be the number of disks in $D^I$.
\begin{lemma}
Let $L_V^{*}$ be the length of the optimal tour that visits at least one point in each visibility region, $V_i$. We have: $L_V^{*} \geq \dfrac{m}{2}\alpha h\tan\delta$ where $\alpha=0.4786$ and when $m\geq 3$.
\end{lemma}
\begin{proof}
From Theorem 1 in~\cite{tekdas2012efficient}, we know that any tour of length $L$ that visits at least one point in $m$ disjoint disks of radius $r$ satisfies,
\begin{equation}
L \geq \dfrac{m}{2} \alpha r.
\end{equation}
Therefore, the optimal tour of length $L_{D_I}^{*}$ that visits all the $m$ disks in $D^I$ must satisfy: 
\begin{equation}
L_{D_I}^{*} \geq \dfrac{m}{2} \alpha h\tan\delta.
\end{equation}
Since $D^I\subseteq D$, the optimal tour that visits at least one point in each disk in $D$ will have a length:
\begin{equation}
L_D^{*} \geq L_{D_I}^{*} \geq \dfrac{m}{2} \alpha h\tan\delta.
\end{equation}
From the above equation and Lemma~\ref{lem:lengthlowerbound}, we get the desired inequality: $L_V^{*} \geq \dfrac{m}{2} \alpha h\tan\delta$.
\end{proof}

\paragraph*{Upper Bound} So far, we have only presented a lower bound on the length of the shortest tour that visits each disk $D_i$. Note that visiting each disk in $D_i$ is necessary to visit $V_i$ but may not be sufficient. Instead, we will replace each $V_i$ by an inner disk, say $d_i$, such that it is completely contained within $V_i$, i.e., $d_i\subseteq V_i$. The inner disk $d_i$ is also centered at the same point as $D_i$ and all inner disks have the same radius.

Let $l$ be the maximum height of the terrain. That is, all points on the terrain are at height of $l$ or below. It is easy to see that $l<h$. We set the radius of the inner disks to be equal to $(h-l)\tan\delta$. Using a similar argument as given in Lemma~\ref{lem:contained}, we can prove that $d_i\subseteq V_i$. That is, all inner disks are completely contained within the visibility regions, $V_i$.

Our algorithm for solving the route planning problem is to find a tour that visits at least one point in each inner disk, $d_i$. This can be found using the algorithm presented by Dumitrescu and Mitchell~\cite{dumitrescu2003approximation}. (1) Find the maximum independent set of non-overlapping inner disks, say $d^I$. (2) Find a $(1+\epsilon)$--approximation to the optimal TSP tour that visits the center of all disks in $d^I$. (3) Follow the tour found in the second step. Every time the tour enters a new disk, take a detour to follow the circumference till you reach the same point again, and then move towards the center. Note that this step adds a detour of length at most $2\pi r$ to the TSP tour, where $r$ is the radius of the disk.

Let $L_d$ be the length of this tour. We now provide an upper bound to the length of this tour.
\begin{lemma}
Let $L_d$ be the length of the tour found using the proposed algorithm. We have: $$L_d \leq (1+\epsilon) (L_{V}^{*} + 2m_d(h-l)\tan\delta) + 2m_d\pi (h-l)\tan\delta),$$ where $m_d$ is the maximum number of non-overlapping inner disks, $d^I$.
\label{lem:lengthupperbound}
\end{lemma}
\begin{proof}
The length of the tour, $L_d$, is equal to the distance traveled to visit the centers of the disk in $d^I$ (Step 2) and the detours added every the center is visited (Step 3). Let $L_{TSP}^{*}$ be the length of the optimal TSP tour that visits the center of the disks in $d^I$. Although finding $L_{TSP}^{*}$ is NP-hard, there exists polynomial time approximation algorithms that find a tour whose length is at most $(1+\epsilon)L_{TSP}^{*}$ for any $\epsilon>0$. Therefore,
\begin{align}
L_d &\leq (1+\epsilon) (L_{TSP}^{*}) + 2m_d\pi (h-l)\tan\delta\\
&\leq (1+\epsilon) (L_{V}^{*} + 2m_d (h-l)\tan\delta) + 2m_d\pi (h-l)\tan\delta.
\end{align}
The second inequality follows from the fact that we can always construct a tour that visits the center of the disks in $d^I$ by first finding the optimal tour that visits at least one point in each disk in $d^I$ (of length $L_{d}^{*}$) and then adding a detour of at most $2r$ to visit the center. That is, $L_{TSP}^{*} \leq L_{d}^{*} + 2m_d$ and $L_d^{*} \leq L_V^{*}$.
\end{proof}

What remains to show is the relationship between $m$ and $m_d$. It is easy to see that $m_d \geq m$, that is the number of non-overlapping outer disks (in $D$) cannot be more than the number of non-overlapping inner disks ($d_i$). We will show that $m_d$ cannot be arbitrarily larger than $m$.

\begin{lemma}
Let $m$ be the maximum number of non-overlapping outer disks, $D_i$. Let $m_d$ be the maximum number of non-overlapping inner disks, $d_i$. We have:
\begin{equation}
m_d \leq \left(\frac{2h}{h-l}\right)^2 m.
\end{equation}
\end{lemma}
\begin{proof}
Consider an outer disk, $D_i$, whose radius is equal to $h\tan\delta$. Draw another disk, say $D_i^\prime$, whose radius is equal to $2h\tan\delta$ with the same center. Any inner disk that intersects with $D_i$ is completely contained within $D_i^\prime$. We now bound the maximum number of inner disks that can be packed within $D_i^\prime$ without any two overlapping. One inner disk has an area of $\pi (h-l)^2\tan^2\delta$. Therefore, at most $\left(\frac{2h}{h-l}\right)^2$ non-overlapping inner disks are contained within $D_i^\prime$. Since there are $m$ non-overlapping outer disks, we get the desired result.
\end{proof}

We are now ready to state the main result of this section.
\begin{theorem}
Let $L_d$ be the length of the tour found using the proposed algorithm. Let $L_V^{*}$ be the length of the optimal tour that visits at least one point in each visibility region, $V_i$. We have:
\begin{equation}
L_d \leq \left((1+\epsilon) \left( 1 + 16 \frac{h}{\alpha(h-l)} L_V^{*}\right) + 16\pi \frac{h}{\alpha(h-l)} \right) L_V^{*},
\end{equation}
where $h$ is the height of the fixed-altitude plane, $l$ is the height of the tallest point on the terrain, and $\alpha=0.4786$.
\end{theorem}
\begin{proof}
We know from Lemma~\ref{lem:lengthupperbound},
\begin{align*}
L_d &\leq (1+\epsilon) \left(L_{V}^{*} + 2m_d(h-l)\tan\delta\right) + 2m_d\pi (h-l)\tan\delta,\\
&\leq (1+\epsilon) \left(L_{V}^{*} + 2\left(\frac{2h}{h-l}\right)^2m (h-l) \tan\delta)\right)\\ 
&+ 2 \left(\frac{2h}{h-l}\right)^2m \pi (h-l)\tan\delta,\\
&\leq (1+\epsilon) \left(L_{V}^{*} + 2\left(\frac{2h}{h-l}\right)^2 \frac{2}{\alpha h \tan\delta} L_V^{*} (h-l) \tan\delta)\right)\\ 
&+ 2 \left(\frac{2h}{h-l}\right)^2\frac{2}{\alpha h \tan\delta} L_V^{*}  \pi (h-l)\tan\delta,\\
&\leq \left((1+\epsilon) \left( 1 + 16 \frac{h}{\alpha(h-l)}\right) + 16\pi \frac{h}{\alpha(h-l)} \right) L_V^{*}\\
&\leq O(1) L_V^{*}.
\end{align*}
\end{proof}

This shows that the proposed algorithm yields a constant-factor approximation. The constant depends on two parameters, maximum height of the terrain and the height of the fixed-altitude plane, but is otherwise independent of the input (e.g., $|\mathcal{P}|$, the width of the terrain, etc.). We would like to remark that, the same approximation algorithm may also be used in the case of constant resolution imagery (variable flight altitude) missions to compute UAV tours within a constant-factor of the optimal. It is easy to see that the enclosing outer disks ($D_i$) and enclosed inner disks ($d_i$), used to compute the lower and upper bounds on the UAV tour respectively, are still valid and may be used to compute the same constant approximation factor.
\section{Route Planning}\label{sec:routePlanning}
The visibility regions for all points of interest computed as discussed in Section \ref{sec:visComputation} are given as input to the route planning stage. Unless all visibility regions are contained within one of the visibility region, in which case the containing region may as well be ignored and the problem be solved for $m-1$ visibility regions, the tour for the aerial robot must enter each visibility region at a \emph{point on the boundary} of the region. This implies that we can restrict the search for the points visited by the aerial robot to the boundaries of the respective visibility regions. Similar ideas have been used by Obermeyer et al. \cite{obermeyer2012sampling} for path planning for a non-holonomic robot through a set of polygonal spaces. Hence, we consider only the points on the boundary of each visibility region to compute a tour for the aerial robot. Each visibility region contributes $d$ (sampling parameter, refer Section \ref{sec:visComputation}) unique points on its boundary. To address the overlapping regions case, we duplicate all points in the overlapping regions and add them to each region in the intersection.

To formalize, 
let $S_i$ be the set of sample points on the boundary of the $i^{th}$ visibility region corresponding to the $p_i$ point of interest and $\mathcal S = \{S_i: i \in [1,m]\}$ be the set of all such sets. Let $\mathcal V = \bigcup\limits _{i=1}^{m} S_{i}$ be the set of all vertices on boundary of a visibility region. $\mathcal V$ also includes duplicate points that lie in the intersection of visibility regions. We define the cost function, $c_{ij}:\mathcal V\times \mathcal V \rightarrow \mathbb R^+$, as the length of path for the aerial robot to go from $v_i$ to $v_j$, where $v_i, v_j \in \mathcal V$. In this form, the problem reduces to an instance of the well-known Generalized Traveling Salesman Problem (GTSP). We employ two strategies to solve the GTSP instance. An Integer Linear Programming (ILP) formulation solved within a branch-and-cut framework and a specialized GTSP solver called GLNS \cite{Smith2016GLNS}. We give the ILP formulation below  and describe the GLNS solver settings in the next Section.

\subsection*{ILP Formulation}
To formulate the problem as an Integer Linear Program, we define binary decision variable, $y_{ij}$, for each pair of vertices $v_i$ and $v_j$ in the set $\mathcal V$. $y_{ij} = 1$ if the aerial robot visits $v_i$ and $v_j$ vertices in order. Let $\delta^+(X)$ denote the set of pairs $(i,j)$ such that $v_i \in X$ and $v_j \in \mathcal V\setminus{X}$ and $\mathbb P(X)$ denote the power set of $X$. The objective function and constraints of the ILP formulation are defined as 

\noindent \emph{Objective:}
\begin{flalign}
&~ \min \sum \limits_{v_i \in \mathcal V} \sum \limits_{v_j \in \mathcal V} c_{ij} y_{ij} & \label{eq:obj} 
\end{flalign}

\noindent \emph{Degree Constraints:}
\begin{flalign}
& \sum_{v_j \in \mathcal V \setminus{v_i}} y_{ji} - \sum_{v_j\in \mathcal V \setminus{v_i}} y_{ij} = 0, \quad \forall v_i \in \mathcal V \label{constr:tourConst} &\\
& \sum_{v_i \in S_k} \sum_{v_j \in \mathcal V \setminus S_k} y_{ji} = 1, \quad \forall k \in [1\ldots m] \label{constr:visitConst} &
\end{flalign}

\noindent \emph{Sub-tour Elimination Constraints:}
\begin{flalign}
&\sum_{(i,j) \in \delta^+(s)} y_{ij} = 1, \quad\forall s \in \mathbb P(\mathcal S)\setminus \{\mathcal S,\phi\} &\label{constr:subTour}
\end{flalign}

\noindent \emph{Variable Domain:}
\begin{flalign}
& y_{ij} \in \{0,1\} \quad \forall v_i, v_j \in \mathcal V \label{constr:yDomain} &
\end{flalign}

Equations \eqref{constr:tourConst} and \eqref{constr:visitConst} represent tour constraints and ensure each visibility region is visited. Equation \eqref{constr:subTour} represents the set of sub-tour elimination constraints. The number of sub-tour elimination constraints grows exponentially with increase in the number of visibility regions. Therefore, we employ a branch-and-cut strategy to solve the ILP formulation. A relaxed formulation, minus the sub-tour elimination constraints is given as input to the solver. A separation algorithm (Algorithm \ref{algo:sepAlgo}) computes valid inequalities (given by Equation \eqref{constr:subTourSepAlgo}) at runtime and adds them to the formulation to ensure feasibility of the final solution. 
Branch-and-cut has been observed to be an effective strategy to improve computational time in problems of similar flavor \cite{grotschel1985}.
\begin{flalign}
&\sum_{(i,j) \in \delta^+(\kappa)} y_{ij} = 1. &\label{constr:subTourSepAlgo}
\end{flalign}


 \begin{algorithm}[t]
 \caption{Separation Algorithm}\label{algo:sepAlgo}
 \begin{algorithmic}
\State Build graph $G$(directed) $\equiv (\mathcal P,E)$
\State Add edge $(i,j)$ to E, if $\exists~ v_k \in S_i, v_l\in S_j ~{\rm and }~ y_{kl} = 1$
\State Find connected components $\mathcal G$ in $G$
\If{$(|\mathcal G|>1)$}
\ForAll{connected components $\kappa \in \mathcal G$}
\State Add valid inequality (Eq. \eqref{constr:subTourSepAlgo})
\EndFor
\EndIf
 \end{algorithmic}
 \end{algorithm}


\section{Simulation Results}\label{sec:simResults}

\begin{figure}
\centering
\begin{subfigure}{0.48\linewidth}
\centering
\includegraphics[scale=0.4]{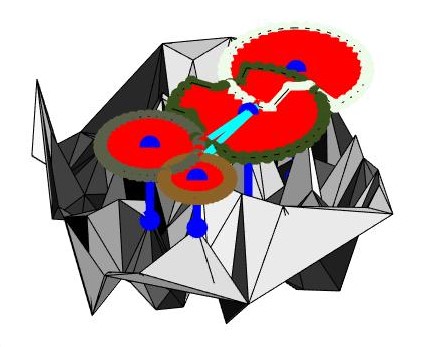}
\caption{}
\end{subfigure}
\begin{subfigure}{0.48\linewidth}
\centering
\includegraphics[scale=0.4]{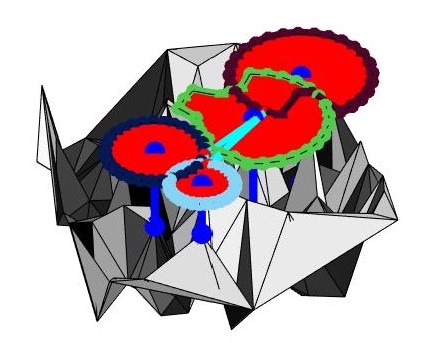}
\caption{}
\end{subfigure}
\caption{Aerial robot tours generated by (a) GLNS solver and (b) ILP solver. The tour is shown in cyan color. In both the cases, the tour is the same for this particular instance.}\label{fig:samplePathSim}
\end{figure}


The performance of the two stage strategy is evaluated using IBM ILOG CPLEX library (version 12.7) in C++11 and GLNS solver in Julia \cite{Smith2016GLNS}. For visualization, we use MATLAB R2017a with TIN based modeling of the terrain.

\begin{table*}
\centering
\begin{tabular}{lccccccccc}
\multicolumn{10}{c}{Number of instances solved by the ILP solver.}\\
\toprule
\multirow{2}{*}{$m$} &\multicolumn{3}{c}{$\delta = 20\degree$} &\multicolumn{3}{c}{$\delta = 30\degree$} & \multicolumn{3}{c}{$\delta = 40\degree$}\\
&d = 20  &d = 30  &d = 40&d = 20  &d = 30  &d = 40&d = 20  &d = 30  &d = 40 \\
\midrule
4 & 20 (20)  & 20 (20) & 9 (20) & 20 (20)  & 19 (20) & 7 (20) & 20 (20)  & 18 (20) & 13 (20)\\
6 & 14 (20)  & 0 (20)  & 0 (20) & 10 (20)  & 0 (20)  & 0 (19) & 7 (20)  & 0 (20)  & 0 (19)\\
8 & 0 (20)   & 0 (20)  & 0 (18) & 0 (19)   & 0 (17)  & 0 (8) & 0 (18)   & 0 (13)  & 0 (8)\\
\bottomrule
\end{tabular}%
\caption{The table shows the number of instances solved optimally by the ILP solver. Numbers in bracket represent the number of instances for which the solver could compute at least a feasible solution.} \label{tab:ILPnum}
\end{table*}

\begin{table*}
\centering
\begin{tabular}{lccccccccc}
\multicolumn{10}{c}{Mean percentage relative gap of GLNS solver solution w.r.t. ILP generated lower bounds.}\\
\toprule
\multirow{2}{*}{$m$} &\multicolumn{3}{c}{$\delta = 20\degree$} &\multicolumn{3}{c}{$\delta = 30\degree$} & \multicolumn{3}{c}{$\delta = 40\degree$}\\
&d = 20  &d = 30  &d = 40&d = 20  &d = 30  &d = 40&d = 20  &d = 30  &d = 40 \\
\midrule
4 & 0 (0)  & 0 (0) & 4.8 (5.6) & 0 (0)  & 1 (4.5) & 12.4 (16.7) & 0 (0)  & 2.1 (7.4) & 14 (24.7)\\
6 & 4.5 (9.5)  & 24.7 (11.2)  & 30.1 (12.2)  & 17.8 (23.1)  & 44.8 (16.8)  & 52.3 (18.2)  & 28.4 (27.9)  & 57.3 (18)  & 68 (19.8)\\
8 & 31.6 (15.2)   & 41.7 (17.8)  & 48.2 (23.2)  & 53.1 (20.8)   & 68.2 (20)  & 85.1 (20.2) & 76.2 (19.7)   & 85 (16.1)  & 94 (10.2)\\
\bottomrule
\end{tabular}%
\caption{Table shows the mean percentage relative gap of GLNS solutions w.r.t. ILP solver generated lower bounds. The numbers in bracket represent the standard deviation.} \label{tab:GLNSgap}
\end{table*}


\begin{table*}
\centering
\begin{tabular}{lccccccccc}
\multicolumn{10}{c}{Mean percentage relative gap of ILP generated solutions.}\\
\toprule
\multirow{2}{*}{$m$} &\multicolumn{3}{c}{$\delta = 20\degree$} &\multicolumn{3}{c}{$\delta = 30\degree$} & \multicolumn{3}{c}{$\delta = 40\degree$}\\
&d = 20  &d = 30  &d = 40&d = 20  &d = 30  &d = 40&d = 20  &d = 30  &d = 40 \\
\midrule
4 & 0 (0)  & 0 (0) & 5.3 (5.9) & 0 (0)  & 1 (4.7) & 13.0 (16.9) & 0 (0)  & 2.2 (7.7) & 14.2 (24.9)\\
6 & 4.7 (9.8)  & 25.9 (10.7)  & 32 (11.4)  & 18.2 (23.4)  & 45.3 (16.8)  & 51.3 (14.5)  & 28.9 (28.1)  & 58.5 (17.8)  & 68.5 (18.7)\\
8 & 33.7 (14.3)   & 43.7 (17.6)  & 45.3 (14.8)  & 52.3 (18.2)   & 65.1 (16.3)  & 65.7 (13) & 74.99 (18.4)   & 79.6 (14.3)  & 87.5 (9.3)\\
\bottomrule
\end{tabular}%
\caption{The table shows the relative gap for best ILP solutions  w.r.t. solver generated lower bound within a maximum time limit of 900 seconds. The numbers in bracket represent the standard deviation.} \label{tab:ILPgap}
\end{table*}

\subsection{Simulation setup}\label{sec:simSetup}

To generate the simulation instances we use an environment of size 200 units$\times$200 units. A 10$\times$10 grid was superimposed on the environment and terrain altitude at each grid point was sampled randomly between 0 to 100 units. A TIN representation of the terrain was then generated by a piecewise triangular interpolation between neighboring grid points. 20 different terrains were generated using this method. The aerial robot flight altitude was fixed at 125 units (clear of all terrain features). The size of the set $\mathcal P$ of points to be monitored on the terrain was varied from 4 to 8 in steps of 2. The camera FOV view angle $\delta$ was varied from 20$\degree$ to 40$\degree$ in steps of 10. The value of sampling resolution parameter $d$, that determines the number of points sampled on the boundary of the terrain, was varied from 20 to 40 in steps of 10. A total of 540 instances were generated.

The total number of points in the GTSP instances were in the range of 80 to 1500.
In general, the GTSP instance size increases with increase in number of points (sets) and the value of sampling parameter $d$.  It also increases with increase in camera view angle, $\delta$; as this increases the overlap between visibility regions. The GLNS solver was used in the \emph{fast} mode setting and allowed a maximum time of 100 seconds. The maximum time limit was not reached for any of the instances with GLNS. The ILP formulation was implemented in a branch-and-cut framework using the lazy callback functionality of IBM ILOG CPLEX library. The solver was allowed to run for a maximum time of 900 seconds for each instance.

\subsection{Results}\label{sec:results}
Simulation results using GLNS and ILP solvers for GTSP instances generated using visibility regions to compute tours for the aerial robot are summarized in Tables \ref{tab:GLNSgap}-\ref{tab:ILPgap}. GLNS solver found a feasible solution for each instance in under 5 seconds. Table \ref{tab:ILPnum} gives the number of instances solved by the ILP solver in 900 seconds. The ILP solver was not able to find any feasible solution for a sizable number of instances. This is attributed to the time limit of 900 seconds imposed on the ILP solver. We do not report results for larger instances ($m=10$, $d = 50$) for both ILP and GLNS, as the solver could not find the optimal solution for any instance with 8 sets ($m = 8$). The relative gap for both GLNS and ILP solver generated solutions rises quickly with increase in the value of $m$ and $\delta$ (Tables \ref{tab:GLNSgap} and \ref{tab:ILPgap}). This points to the hardness and large size of the instances, as GLNS is a widely used tool to solve GTSP instances. Sample paths generated using the two solution methods are shown in Figure \ref{fig:samplePathSim}.

\section{Field Experiments}\label{sec:expts}

\begin{figure}
\centering
\begin{subfigure}{0.47\linewidth}
\includegraphics[scale=0.18]{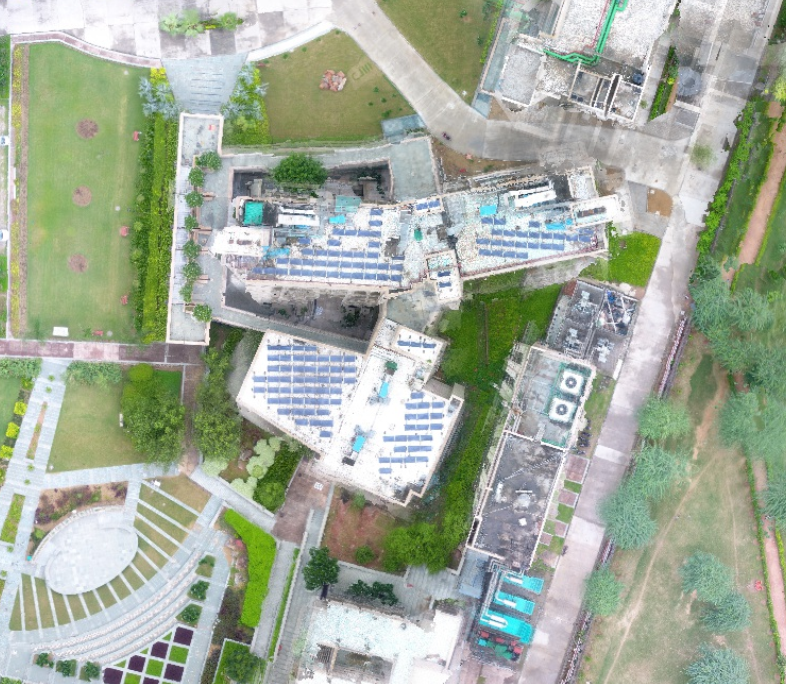}
\caption{}
\end{subfigure}
\begin{subfigure}{0.49\linewidth}
\includegraphics[scale=0.18]{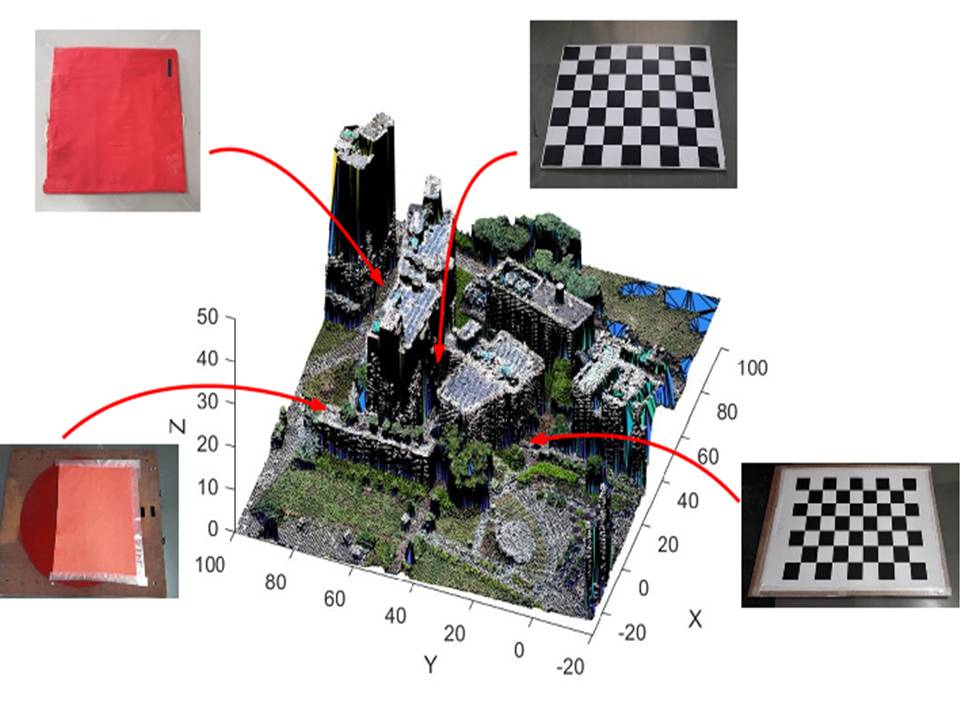}
\caption{}
\end{subfigure}
\caption{(a) Operational area for the field trials. (b) DEM of the topography showing the placement of targets.}
\label{fig:area_DEM}
\end{figure}

Field experiments were conducted at IIIT-Delhi campus to validate the proposed solution approach. The operational area for the experiments is shown in Figure \ref{fig:area_DEM}(a). A DJI Phantom 4 quadrotor was used to perform the experiments. A DEM of the area was created using PIX4D and modeled using a TIN representation. Four target points were placed in the area as shown in Figure \ref{fig:area_DEM}(b). Visibility regions for each target were computed using the visibility computation strategy given in Section \ref{sec:visComputation} (Figure \ref{fig:regions_path}(a)). The value of the sampling resolution parameter $d$ was set to 30. Paths for the UAV were computed as solution to the corresponding GTSP instance solved using GLNS solver, as shown in Figure \ref{fig:regions_path}(b). Experiment footage showing the UAV flight and camera imagery may be viewed in the video attachment to the paper.

\begin{figure}
\centering
\begin{subfigure}{0.48\linewidth}
\centering
\includegraphics[scale=0.195]{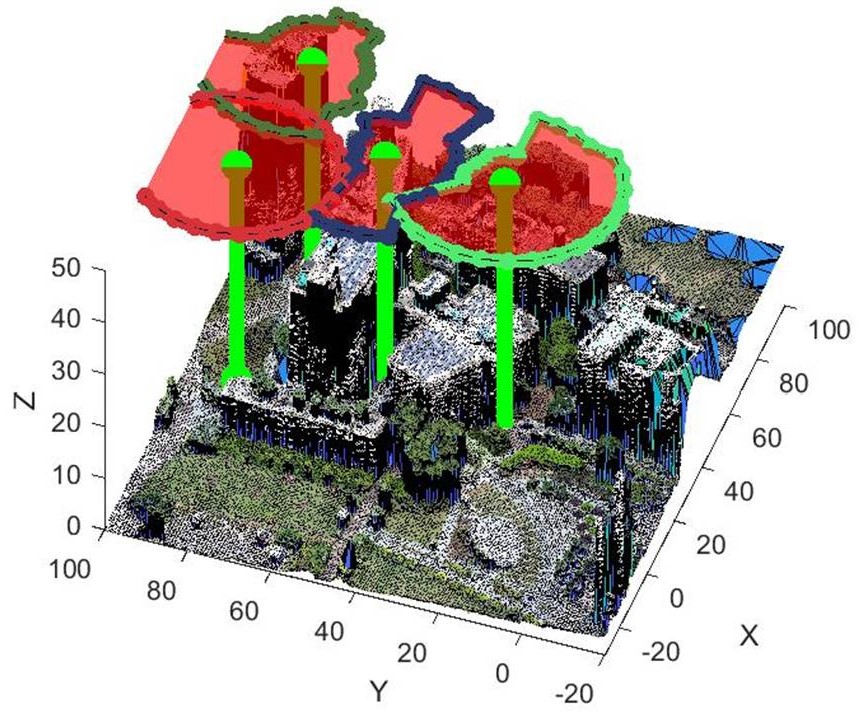}
\caption{}
\end{subfigure}
\begin{subfigure}{0.48\linewidth}
\centering
\includegraphics[scale=0.18]{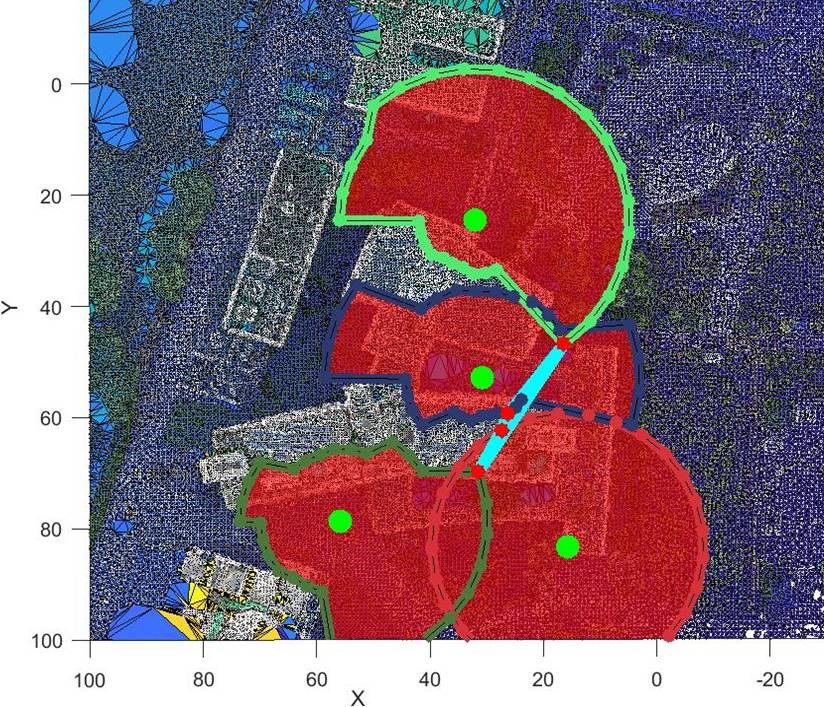}
\caption{}
\end{subfigure}
\caption{(a) Visibility regions for the monitoring targets computed using the strategy discussed in Section \ref{sec:visComputation}. (b) Top view of the operational area showing the UAV path in cyan color and visibility regions.}
\label{fig:regions_path}
\end{figure}
\section{Conclusion and Future Work}\label{sec:concFut}

In this paper, we proposed a two phase strategy to computer tours for a UAV to perform multi-point visual monitoring on a terrain. The path planning problem was modeled as an instance of the TSPN problem and a constant-factor polynomial time approximation algorithm was designed for the class of TSPN instances. Further, GTSP based solution methods by discretizing the visiblity region boundaries were evaluated using two solution techniques -- ILP formulation implemented in a branch-and-cut framework and GLNS (widely used GTSP solver). Field experiments were also conducted to verify the applicability and effectiveness of solution methods. 

A natural extension of the proposed framework is to increase the number of vehicles and optimize the use of vehicles for monitoring larger number of points of interest. Additional venues for future exploration can be (i) design of efficient heuristics to decrease the tour cost (ii) persistent monitoring problem wherein fuel limitations of the aerial robot need to be addressed  and (iii) exploring regions that have non-convex terrains. 


\begin{thebibliography}{99}
\bibitem{tokekar2016precision}
P.~Tokekar, J.~Vander~Hook, D.~Mulla, and V.~Isler, ``Sensor planning for a
symbiotic uav and ugv system for precision agriculture,'' \emph{IEEE
	Transactions on Robotics}, vol.~32, no.~6, pp. 1498--1511, 2016.

\bibitem{araujo2013coverage}
J.~F. Araújo, P.~B. Sujit, and J.~B. Sousa, ``Multiple uav area decomposition
and coverage,'' in \emph{Symposium on Computational Intelligence for Security
	and Defense Applications}, 2013, pp. 30--37.

\bibitem{mainiRefuelUAV}
P.~Maini, K.~Sundar, M.~Singh, S.~Rathinam, and P.~B. Sujit, ``Cooperative
aerial-ground vehicle route planning with fuel constraints for coverage
applications,'' \emph{IEEE Transactions on Aerospace and Electronic Systems},
2019.

\bibitem{obermeyer2012sampling}
K.~J. Obermeyer, P.~Oberlin, and S.~Darbha, ``Sampling-based path planning for
a visual reconnaissance unmanned air vehicle,'' \emph{Journal of Guidance,
	Control, and Dynamics}, vol.~35, no.~2, pp. 619--631, 2012.

\bibitem{morgenthal2014quality}
G.~Morgenthal and N.~Hallermann, ``Quality assessment of unmanned aerial
vehicle (uav) based visual inspection of structures,'' \emph{Advances in
	Structural Engineering}, vol.~17, no.~3, pp. 289--302, 2014.

\bibitem{erdelj2017help}
M.~Erdelj, E.~Natalizio, K.~R. Chowdhury, and I.~F. Akyildiz, ``Help from the
sky: Leveraging uavs for disaster management,'' \emph{IEEE Pervasive
	Computing}, no.~1, pp. 24--32, 2017.

\bibitem{sujit2007cooperative}
P.~Sujit, D.~Kingston, and R.~Beard, ``Cooperative forest fire monitoring using
multiple uavs,'' in \emph{IEEE Conference on Decision and Control}.\hskip 1em
plus 0.5em minus 0.4em\relax IEEE, 2007, pp. 4875--4880.

\bibitem{mitchell1999guillotine}
J.~S. Mitchell, ``Guillotine subdivisions approximate polygonal subdivisions: A
simple polynomial-time approximation scheme for geometric tsp, k-mst, and
related problems,'' \emph{SIAM Journal on Computing}, vol.~28, no.~4, pp.
1298--1309, 1999.

\bibitem{priorityBasedRouting}
V.~K. Shetty, M.~Sudit, and R.~Nagi, ``Priority-based assignment and routing of
a fleet of unmanned combat aerial vehicles,'' \emph{Computers and Operations
	Research}, vol.~35, no.~6, pp. 1813 -- 1828, 2008.

\bibitem{maini2018persistent}
P.~Maini, K.~Yu, S.~P.~B., and P.~Tokekar, ``Persistent monitoring with
refueling on a terrain using a team of aerial and ground robots,'' in
\emph{International Conference on Intelligent Robots and Systems}, 2018.

\bibitem{maini2018visibility}
P.~Maini, G.~Gupta, P.~Tokekar, and S.~P.~B., ``Visibility-based monitoring of
a path using a heterogeneous robot team,'' in \emph{International Conference
	on Intelligent Robots and Systems}, 2018.

\bibitem{smithPersistent2015}
N.~Mathew, S.~L. Smith, and S.~L. Waslander, ``Multirobot rendezvous planning
for recharging in persistent tasks,'' \emph{IEEE Transactions on Robotics},
vol.~31, no.~1, pp. 128--142, Feb 2015.

\bibitem{mainiASCC2017}
P.~Maini, S.~Rathinam, and P.~B. Sujit, ``Curvature constrained trajectory
planning for a uav through a sequence of points: A perturbation approach,''
in \emph{Asian Control Conference}, 2017, pp. 1276--1281.

\bibitem{xu2011optimal}
A.~Xu, C.~Viriyasuthee, and I.~Rekleitis, ``Optimal complete terrain coverage
using an unmanned aerial vehicle,'' in \emph{International Conference on
	Robotics and Automation}, 2011, pp. 2513--2519.

\bibitem{regionIntervisibility}
E.~Moet, M.~Van~Kreveld, and R.~Van~Oostrum, ``Region intervisibility in
terrains,'' \emph{International Journal of Computational Geometry \&
	Applications}, vol.~17, no.~04, pp. 331--347, 2007. 


\bibitem{visibilitySurvey2003}
L.~Floriani and P.~Magillo, ``Algorithms for visibility computation on
terrains: A survey,'' \emph{Environment and Planning B: Planning and Design},
vol.~30, no.~5, pp. 709--728, 2003.

\bibitem{watchtower_pankaj}
P.~K. Agarwal, S.~Bereg, O.~Daescu, H.~Kaplan, S.~Ntafos, and B.~Zhu,
``Guarding a terrain by two watchtowers,'' in \emph{Symposium on
	Computational Geometry}, 2005.

\bibitem{tan2001fast}
X.~Tan, ``Fast computation of shortest watchman routes in simple polygons,''
\emph{Proc. Information Letters}, vol.~77, no.~1, pp. 27--33, 2001.

\bibitem{Carlsson1999}
S.~Carlsson, H.~Jonsson, and B.~J. Nilsson, ``Finding the shortest watchman
route in a simple polygon,'' \emph{Discrete {\&} Computational Geometry},
vol.~22, no.~3, pp. 377--402, Oct 1999.

\bibitem{tokekar2015persistent}
P.~Tokekar and V.~Kumar, ``Visibility-based persistent monitoring with robot
teams,'' in \emph{IEEE/RSJ International Conference on Intelligent Robots and
	Systems}, Sept 2015, pp. 3387--3394.

\bibitem{histogram}
S.~Carlsson, B.~J. Nilsson, and S.~Ntafos, ``Optimum guard covers and
m-watchmen routes for restricted polygons,'' in \emph{Workshop on Algorithms
	and Data Structures}.\hskip 1em plus 0.5em minus 0.4em\relax Springer, 1991,
pp. 367--378.

\bibitem{hameed2016side}
I.~A. Hameed, A.~la~Cour-Harbo, and O.~L. Osen, ``Side-to-side 3d coverage path
planning approach for agricultural robots to minimize skip/overlap areas
between swaths,'' \emph{Robotics and Autonomous Systems}, vol.~76, pp.
36--45, 2016.

\bibitem{nam2016approach}
L.~Nam, L.~Huang, X.~Li, and J.~Xu, ``An approach for coverage path planning
for uavs,'' in \emph{IEEE international workshop on advanced motion control
	(AMC)}.\hskip 1em plus 0.5em minus 0.4em\relax IEEE, 2016, pp. 411--416.

\bibitem{choi3D}
Y.~{Choi}, Y.~{Choi}, S.~{Briceno}, and D.~N. {Mavris}, ``Three-dimensional uas
trajectory optimization for remote sensing in an irregular terrain
environment,'' in \emph{International Conference on Unmanned Aircraft
	Systems}, June 2018, pp. 1101--1108.

\bibitem{sack1999handbook}
J.-R. Sack and J.~Urrutia, \emph{Handbook of computational geometry}.\hskip 1em
plus 0.5em minus 0.4em\relax Elsevier, 1999.

\bibitem{horizonComputation}
A.~J. Stewart, ``Fast horizon computation at all points of a terrain with
visibility and shading applications,'' \emph{IEEE Transactions on
	Visualization and Computer Graphics}, vol.~4, no.~1, pp. 82--93, Jan 1998.

\bibitem{dumitrescu2003approximation}
A.~Dumitrescu and J.~S. Mitchell, ``Approximation algorithms for tsp with
neighborhoods in the plane,'' \emph{Journal of Algorithms}, vol.~48, no.~1,
pp. 135--159, 2003.

\bibitem{mitchell2010constant}
J.~S. Mitchell, ``A constant-factor approximation algorithm for tsp with
pairwise-disjoint connected neighborhoods in the plane,'' in
\emph{Proceedings of the annual symposium on Computational geometry}.\hskip
1em plus 0.5em minus 0.4em\relax ACM, 2010, pp. 183--191.

\bibitem{tekdas2012efficient}
O.~Tekdas, D.~Bhadauria, and V.~Isler, ``Efficient data collection from
wireless nodes under the two-ring communication model,'' \emph{International
	Journal of Robotics Research}, vol.~31, no.~6, pp. 774--784, 2012.

\bibitem{Smith2016GLNS}
S.~L. Smith and F.~Imeson, ``{GLNS}: An effective large neighborhood search
heuristic for the generalized traveling salesman problem,'' \emph{Computers
	\& Operations Research}, vol.~87, pp. 1--19, 2017.

\bibitem{grotschel1985}
M.~Gr{\"o}tschel, M.~W. Padberg, \emph{et~al.}, ``Polyhedral theory,''
\emph{The traveling salesman problem}, pp. 251--305, 1985.

\end{thebibliography}

\end{document}